\documentclass[submission,copyright,creativecommons]{eptcs}
\providecommand{\event}{TARK 2023} 
\usepackage{amsmath,amssymb,amsthm,relsize,enumitem}
\usepackage{setspace,hyperref}
\usepackage{iftex}

\ifpdf
  \usepackage{underscore}         
  \usepackage[T1]{fontenc}        
\else
  \usepackage{breakurl}           
\fi
\theoremstyle{theorem}

\newtheorem{proposition}{Proposition}
\newtheorem{lemma}{Lemma}
\newtheorem{remark}{Remark}
\newtheorem{definition}{Definition}
\title{Characterization of {AGM} {B}elief {C}ontraction\\ in {T}erms of {C}onditionals}
\author{Giacomo Bonanno\thanks{I am grateful to three anonymous reviewers for their comments.}
\institute{Department of Economics\\
University of California\\
Davis, California, USA}
\email{gfbonanno@ucdavis.edu}
}
\def\titlerunning{Contraction and {C}onditionals}
\def\authorrunning{G. Bonanno}
\begin{document}
\maketitle

\begin{abstract}
We provide a semantic characterization of AGM belief contraction based on frames consisting of a Kripke belief relation and a Stalnaker-Lewis selection function. The central idea is as follows. Let $K$ be the initial belief set and $K\div\phi$ be the contraction of $K$ by the formula $\phi$; then $\psi\in K\div\phi$ if and only if, at the actual state, the agent believes $\psi$ and believes that if $\lnot\phi$ is (were) the case then $\psi$ is (would be) the case.
\end{abstract}
\section{Introduction}
Belief contraction is the operation of  removing from the set $K$ of initial beliefs a particular belief $\phi$. One reason for doing so is, for example, the discovery that some previously trusted evidence supporting $\phi$ was faulty. For instance, a prosecutor might form the belief that the defendant is guilty on the basis of his confession; if the prosecutor later discovers that the confession was extorted, she might abandon the belief of guilt, that is, become open minded about whether the defendant is guilty or not. In their seminal contribution to belief change, Alchourr\'{o}n, G\"{a}rdenfors and Makinson (\cite{AGM85}) defined the notion of "rational and minimal" contraction by means of a set of eight properties, known as the AGM axioms or postulates. They did so within a syntactic approach where the initial belief set $K$ is a consistent and deductively closed set of propositional formulas and the result of removing $\phi$ from $K$ is a new set of propositional formulas, denoted by $K\div \phi$.
\par
We provide a new characterization of AGM belief contraction based on a so-far-unnoticed connection between the notion of belief contraction and the Stalnaker-Lewis theory of conditionals (\cite{Stal68,Lew73}). Stalnaker introduced the notion of a selection function $f$ taking as input a possible world $w$ and a set of worlds $E$ (representing a proposition) and giving as output a world $w'=f(w,E)\in E$, interpreted as the closest $E$-world to $w$ (an $E$-world is a world that belongs to $E$). Lewis generalized this by allowing $f(w,E)$ to be a set of worlds. In the Stalnaker-Lewis theory the (indicative or subjunctive) conditional "if $\phi$ is (were) the case then $\psi$ is (would be) the case", denoted by $\phi>\psi$, is declared to be true at a world $w$ if and only if $\psi$ is true at all the worlds in $f(w,\Vert\phi\Vert)$ ($\Vert\phi\Vert$ denotes the set of worlds at which $\phi$ is true).
\par
We consider semantic frames consisting of a Kripke belief relation on a set of states $S$, representing the agent's initial beliefs, and a Stalnaker-Lewis selection function on $S\times 2^S$ representing conditionals. Adding a valuation to such a frame yields a model. Given a model, we define the initial belief set $K$ as the set of formulas that the agent believes at the actual state and $K\div \phi$ (the contraction of $K$ by $\phi$) as the set of formulas that the agent believes initially and also on the supposition that $\lnot\phi$: \,$\psi\in K\div \phi$ if and only if, at the actual state, the agent (1) believes $\psi$ and (2) believes the conditional $\lnot\phi>\psi$. We show that, when the selection function satisfies some natural properties, the contraction operation so defined captures precisely the set of AGM belief contraction functions.
\section{AGM contraction functions}
\label{SEC:AGM}
Let \texttt{At} be a countable set of atomic formulas. We denote by $\Phi_0$ the set of Boolean formulas constructed from  \texttt{At} as follows: $\texttt{At}\subset \Phi_0$ and if $\phi,\psi\in\Phi_0$ then $\neg\phi$ and $\phi\vee\psi$ belong to $\Phi_0$. Define  $\phi\rightarrow\psi$, $\phi\wedge\psi$, and $\phi\leftrightarrow\psi$ in terms of $\lnot$ and $\vee$ in the usual way.
\par
Given a subset $K$ of $\Phi_0 $, its deductive closure $Cn(K)\subseteq\Phi_0$ is defined as follows: $\psi \in Cn(K)$ if and only if there exist $\phi _1,...,\phi _n\in K$ \ (with $n\geq 0$) such that $(\phi _1\wedge ...\wedge \phi _n)\rightarrow \psi $ is a tautology. A set $K\subseteq \Phi_0 $ is \textit{consistent} if $ Cn(K)\neq \Phi_0 $; it is \textit{deductively closed} if $K=Cn(K)$. Given a set $K\subseteq \Phi_0$ and a formula $\phi\in\Phi_0$, the \emph{expansion} of $K$ by $\phi$, denoted by $K+\phi$, is defined as follows: $K+\phi=Cn\left(K\cup\{\phi\}\right)$.
\par
Let $K\subseteq\Phi_0$ be a consistent and deductively closed set representing the agent's initial beliefs and let $\Psi \subseteq \Phi_0 $ be a set of formulas representing possible candidates for withdrawal. A \emph{belief contraction function} (based on $K$ and $\Psi$) is a function $\div _\Psi:\Psi \rightarrow 2^{\Phi_0 }$ (where $2^{\Phi_0}$ denotes the set of subsets of $\Phi_0 $) that associates with every formula $\phi \in \Psi $ a set $ K\div _\Psi\,\phi \subseteq \Phi_0 $ (interpreted as the result of removing $\phi$ from $K$). If $\Psi \neq \Phi_0 $ then $\div _\Psi$ is called a \emph{partial} contraction function, while if $\Psi =\Phi_0 $ then $\div _{\Phi_0}$ is called a \emph{full-domain} contraction function; in this case we simplify the notation and omit the subscript $\Phi_0$.
\begin{definition}
\label{brf extension}
Let $\div _\Psi:\Psi \rightarrow 2^{\Phi_0}$ be a partial contraction function and $\div ':\Phi_0 \rightarrow 2^{\Phi_0}$ a full-domain contraction function (both of them based on $K$). We say that $\div '$ is an \emph{extension} of $\div _\Psi$ if, for every $\phi \in \Psi$, $ K\div '\phi= K\div _\Psi\phi $.
\end{definition}
\par
A \textit{full-domain} contraction function is called an \textit{AGM contraction function} if it satisfies the following properties, known as the AGM postulates:

\begin{center}
\begin{tabular}{ll}
($K-$1) & [Closure] $K\div \phi=Cn(K\div \phi)$. \\
($K-$2) & [Inclusion] $K\div \phi\subseteq K$. \\
($K-$3) & [Vacuity] If $\phi\notin K$ then $K\subseteq K\div \phi$.\\
($K-$4) &  [Success] If $\phi$ is not a tautology, then $\phi\notin K\div \phi$. \\
($K-$5) & [Recovery] If $\phi\in K$ then $K\subseteq (K\div \phi)+\phi$. \\
($K-$6) & [Extensionality] If $\phi \leftrightarrow \psi$ is a tautology, then $K\div \phi=K\div \psi$. \\
($K-$7) & [Conjunctive overlap] $(K\div \phi)\cap (K\div \psi) \subseteq K\div (\phi \wedge \psi)$. \\
($K-$8) & [Conjunctive inclusion] If $\phi \notin K\div (\phi \wedge \psi)$, then $K\div (\phi \wedge \psi)\subseteq K\div \phi.$
\end{tabular}
\end{center}
\par\noindent
($K-$1) requires the result of contracting $K$ by $\phi$ to be a deductively closed set.\\[3pt]
($K-$2) requires the contraction of $K$ by $\phi$ not to contain any beliefs that were not in $K$.\\[3pt]
($K-$3) requires that if $\phi$ is not in the initial belief set, then every belief in $K$ should also be present in $K\div \phi$ (thus, by ($K-$2) and ($K-$3), if $\phi\notin K$ then the contraction of $K$ by $\phi$ coincides with $K$).\\[3pt]
($K-$4) requires that $\phi$ not be contained in $K\div \phi$, unless $\phi$ is a tautology (in which case, by ($K-$1), it must be in $K\div \phi$). \\[3pt]
($K-$5) is a conservativity requirement: when $\phi\in K$,  contracting by $\phi$ and then expanding the resulting set $K\div \phi$ by $\phi$ should involve no loss of beliefs relative to $K$ (the converse inclusion $(K\div \phi)+\phi\subseteq K$ follows from ($K-$2) and the hypothesis  that $K=Cn(K)$).\\[3pt]
($K-$6) says that logically equivalent formulas should lead to the same result in terms of contraction. \\[3pt]
By ($K-$7), if a formula $\chi\in K$ is neither removed in the contraction of $K$  by $\phi$ nor in the contraction of $K$  by $\psi$, then $\chi$ should not be removed in the contraction of $K$  by the conjunction $\phi\wedge\psi$.\\[3pt]
($K-$8), on the other hand, requires that if  $\phi$ is removed when we contract by $\phi\wedge\psi$, then every formula that survives the contraction of $K$  by $\phi\wedge\psi$ survives also when $K$ is contracted by $\phi$ alone.
\par
For an extensive discussion of the above postulates see \cite{Gae88,Hans99c,FerHans18}.
\par
The notion of AGM belief contraction has been given alternative characterizations. One characterization is in terms of a binary relation $\leqslant$ of "epistemic entrenchment" on $K$, with the interpretation of $\phi\leqslant\psi$ as "$\phi$ is either less entrenched than, or as entrenched as, $\psi$". G\"{a}rdenfors (\cite[Theorem 4.30, p. 96]{Gae88}) shows that if the relation $\leqslant$ satisfies five properties and a contraction function is defined by `$\psi\in K\div \phi$ if and only if $\psi\in K$ and either $\phi$ is a tautology or  $\phi<(\phi\vee\psi)$', then such contraction function is an AGM contraction function and, conversely, if an AGM contraction function is used to define the relation $\leqslant$ by `$\phi\leqslant\psi$ if and only if either  $\phi\notin K\div (\phi \wedge\psi)$ or $\phi \wedge \psi$ is a tautology' then such relation satisfies those five properties. Another characterization makes use of the set $W$ of possible worlds, where a possible world is defined as a maximally consistent set of formulas in $\Phi_0$; within this approach, contraction has been characterized either in terms of systems of spheres (\cite{Gro88,Lew73})  or in terms of a plausibility relation on $W$ or in terms of propositional selection functions (see \cite[Chapter 4]{FerHans18}).
\par
In this paper we provide an alternative characterization in terms of Stalnaker-Lewis conditionals.
\section{An alternative semantic characterization of AGM contraction}
Given a binary relation $R\subseteq S\times S$ on a set $S$, for every $s\in S$ we define $R(s)=\{x\in S: (s,x)\in R\}$.
\begin{definition}
\label{DEF:frame}
A \emph{pointed frame} is a quadruple $\left\langle {S,s_@,\mathcal B,f} \right\rangle$ where
\begin{enumerate}
\item $S$ is a set of \emph{states}; subsets of $S$ are called  \emph{events}.
\item $s_@\in S$ is a distinguished element of $S$ interpreted as the \emph{actual state}.
\item $\mathcal B \subseteq S \times S$ is a binary \emph{belief relation} on $S$ which is serial: $\forall s\in S$, $\mathcal B(s)\neq \varnothing$.
\item $f:\mathcal B(s_@)\times 2^S\setminus\varnothing \rightarrow 2^S$ is a \emph{Stalnaker-Lewis selection function}\footnote{
   Note that, for the purpose of this paper, the domain of $f$ can be taken to be $\mathcal B(s_@)\times 2^S\setminus\varnothing$ rather than $S\times 2^S\setminus\varnothing$. However, it can easily be extended to $S\times 2^S\setminus\varnothing$ as follows: first, fix an arbitrary function $g:S\setminus\mathcal B(s_@)\rightarrow \mathcal B(s_@)$ and then define, for every $s\in S\setminus\mathcal B(s_@)$ and every $\varnothing\ne E\subseteq S$, $f(s,E)=f\left(g(s),E\right)$.
    }\
    that associates with every state-event pair $(s,E)$ (with $s\in\mathcal B(s_@)$ and $\varnothing\ne E\subseteq S$) a set of states $f(s,E)\subseteq S$ such that,
    \begin{enumerate}
  \item (a.1) $f(s,E)\ne\varnothing$ and (a.2) (Success) $f(s,E)\subseteq E$,
  \item (Weak Centering) if $s\in E$ then $s\in f(s,E)$,
  \item (Doxastic Priority 1) if $\mathcal B(s_@)\cap E\neq\varnothing$ then $f(s,E)\subseteq \mathcal B(s_@)\cap E$,
  \item (Intersection) $f(s,E)\cap F\subseteq f(s,E\cap F)$,
  \item (Doxastic Priority 2) Let $B_{EF}=\{s\in\mathcal B(s_@):f(s,E)\cap F\ne\varnothing\}$. If $B_{EF}\ne\varnothing$  then
  \begin{enumerate}
  \item[(e.1)] if $s\in B_{EF}$ then $f(s,E\cap F)\subseteq f(s,E)\cap F$,
  \item[(e.2)] if $s\notin B_{EF}$ then $f(s,E\cap F)\subseteq f(\hat s,E\cap F)$ for some $\hat s\in B_{EF}$.
  \end{enumerate}
\end{enumerate}
\end{enumerate}
\end{definition}
\par
The set $\mathcal B(s)$ is the set of states that the agent considers possible at state $s$, so that $\mathcal B(s_@)$ is the set of doxastic possibilities at the actual state $s_@$ and represents the agent's initial beliefs.  $f(s,E)$ is the set of states that the agent considers closest, or most similar, to state $s$ conditional on event $E$. \\
(4.a) of Definition \ref{DEF:frame} requires $f(s,E)$ to be non-empty and, furthermore, that every state in $f(s,E)$ be an $E$-state. \\
(4.b) postulates that if $s$ is an $E$-state then it belongs to $f(s,E)$, that is, $s$ itself is one of the  $E$-states that are closest to $s$. \\
By (4.c) if there exists an $E$-state among those initially considered possible ($\mathcal B(s_@)\cap E\neq\varnothing$), then, for every $s\in\mathcal B(s_@)$, the closest $E$-states to $s$ must belong to $\mathcal B(s_@)\cap E$.\\
By (4.d), the closest $E$-states to $s$ that are also $F$-states must belong to the set of closest $(E\cap F)$-states to $s$.\\
 (4.e) can be viewed as an extension of (4.c): it says that if, among the states initially considered possible, there is at least one state, call it $s$, that satisfies the property that among its closest $E$-states there is at least one that is also an $F$-state, then (1) the closest $(E\cap F)$-states to $s$ must belong to the intersection $f(s,E)\cap F$ and (2) for any other state that does not satisfy the property, the closest $(E\cap F)$-states to it are contained in the set of closest $(E\cap F)$-states to some state that does satisfy the property.
\smallskip\par
Adding a valuation to a pointed frame yields a model. Thus a \emph{model} is a tuple $\left\langle {S,s_@,\mathcal B,f,V} \right\rangle$ where $\left\langle {S,s_@,\mathcal B,f} \right\rangle$ is a pointed frame and $V:\texttt{At}\rightarrow 2^S$ is a valuation that assigns to every atomic formula $p\in\texttt{At}$ the set of states where $p$ is true.
Given a model $\left\langle {S,s_@,\mathcal B,f,V} \right\rangle$ define truth of a Boolean formula $\phi\in\Phi_0$ at a state $s\in S$, denoted by $s\models\phi$, in the usual way:
\begin{definition}
\label{Truth0}
Truth of a formula at a state is defined as follows:
\begin{enumerate}
\item if $p\in\texttt{At}$ then $s\models p$ if and only if $s\in V(p)$,
\item $s\models\neg\phi$ if and only if $s\not\models\phi$,
\item $s\models\phi\vee\psi$ if and only if $s\models\phi$ or $s\models\psi$ (or both),
\end{enumerate}
\end{definition}
\noindent
We denote by $\Vert\phi\Vert$ the truth set of $\phi$: $\Vert\phi\Vert=\{s\in S: s\models\phi\}$.
\par
Fix a model $M=\left\langle {S,s_@,\mathcal B,f,V} \right\rangle$ and let $K=\{\phi \in \Phi _{0}:\mathcal B(s_@) \subseteq \Vert \phi \Vert \}$ (to simplify the notation, we omit the subscript denoting the model and thus write $K$ rather than $K_M$); thus a Boolean formula $\phi$ belongs to $K$ if and only if at the actual state $s_@$ the agent believes $\phi$. It is shown in the Appendix (Lemma \ref{LEM:KfromModel}) that the set $K\subseteq\Phi_0$ so defined is deductively closed and consistent. Next, for every $\phi\in\Phi_0$ such that $\Vert \lnot\phi \Vert\neq\varnothing$, define $K\div \phi\subseteq\Phi_0$ as follows:
\begin{equation}
\label{EQ:K-}
\begin{array}{*{20}{l}}
 \psi\in K\div \phi\text{ if and only if }&(1)\, \mathcal B(s_@)\subseteq\Vert\psi\Vert,\text{ and } \\
{}&(2)\, \forall s\in\mathcal B(s_@), f\left(s,\Vert\lnot\phi\Vert\right)\subseteq\Vert\psi\Vert.
\end{array}
\end{equation}
\noindent
In \eqref{EQ:K-modal} below we rewrite \eqref{EQ:K-} in an extended language containing a belief operator and a conditional operator, thus making the interpretation more transparent:  $\psi\in K\div \phi$ if and only if, at the actual state $s_@$, the agent believes $\psi$ initially as well as on the supposition that $\lnot\phi$.\footnote{
We take ``believing $\psi$ on the supposition that $\lnot\phi$'' to mean ``believing that if $\lnot\phi$ is (were) the case then $\psi$ is (would be) the case''.
}.
\par
Since, in general, not every $\phi\in\Phi_0$ is such that $\Vert \lnot\phi \Vert\neq\varnothing$, this definition gives rise to a \emph{partial} belief contraction function. The next proposition says that this partial contraction function can be extended to a full-domain AGM contraction function; conversely, given a full-domain AGM contraction function based on a consistent and deductively closed set $K$, there exists a model $M=\left\langle {S,s_@,\mathcal B,f,V} \right\rangle$ such that $K=\{\phi \in \Phi _{0}:\mathcal B(s_@)\subseteq \Vert\phi\Vert\} $ and, for every $\phi\in\Phi_0$ such that $\Vert \lnot\phi \Vert\neq\varnothing$, $K\div \phi$ satisfies \eqref{EQ:K-}. Thus the proposed semantics provides an alternative characterization of  AGM belief contraction. The proof of the following proposition is given in the Appendix.
\begin{proposition}
\label{PROP:char}
\phantom{.}
\begin{enumerate}
  \item[(A)] Given a model $\left\langle {S,s_@,\mathcal B,f,V} \right\rangle$ let $K=\{\phi \in \Phi _{0}:\mathcal B(s_@)\subseteq \Vert\phi\Vert\}$ and, for every $\phi\in\Phi_0$ such that $\Vert\lnot\phi\Vert\ne\varnothing$, let $K\div \phi$ be defined by \eqref{EQ:K-}. Then $K$ is consistent and deductively closed and the (partial) belief contraction function so defined can be extended to a full-domain AGM belief contraction function.
  \item[(B)] Let $K\subset\Phi_0$ be consistent and deductively closed and let $\div :\Phi_0\rightarrow 2^{\Phi_0}$ be an AGM belief contraction function. Then there exists a model $\left\langle {S,s_@,\mathcal B,f,V} \right\rangle$ such that $K=\{\phi \in \Phi _{0}:\mathcal B(s_@)\subseteq \Vert\phi\Vert\} $ and, for every $\phi\in\Phi_0$ such that $\Vert \lnot\phi \Vert\neq\varnothing$, $K\div \phi$ satisfies \eqref{EQ:K-}.
\end{enumerate}
\end{proposition}
The proposed semantics becomes more transparent if we extend the language by introducing two modal operators: a unimodal belief operator $\mathbb B$, corresponding to the belief relation $\mathcal B$, and a bimodal conditional operator $>$, corresponding to the selection function $f$. Recall that $\Phi_0$ is the set of Boolean (or factual) formulas. Let $\Phi_1$ be the modal language constructed as follows.
\begin{itemize}
\item $\Phi_0\subset\Phi_1$,
\item if $\phi,\psi\in\Phi_0$ then $\phi > \psi\in \Phi_1$,
\item all the Boolean combinations of formulas in $\Phi_1$.
\end{itemize}
Thus, for the purpose of this paper, the conditional $\phi > \psi$ (interpreted as the indicative or subjunctive conditional "if $\phi$ is (were) the case then $\psi$ is (would be) the case") is defined only for Boolean formulas.
Finally, let $\Phi$ be the modal language constructed as follows:
\begin{itemize}
\item $\Phi_1\subset\Phi$,
\item if $\phi\in\Phi_1$ then $\mathbb B\phi \in \Phi$,
\item all the Boolean combinations of formulas in $\Phi$.
\end{itemize}
Thus formulas in $\Phi$ are either Boolean or formulas of the form $\phi > \psi$, with $\phi$ and $\psi$ Boolean, or of the form $\mathbb B\phi$ where $\phi$ is either Boolean or of the form $\psi >\chi$ with $\psi$ and $\chi$ Boolean, or a Boolean combination of such formulas.
We can now extend the definition of truth of a formula at a state (Definition \ref{Truth0}) to the set $\Phi$ as follows:
\begin{definition}
\label{Truth1}
If $\phi\in\Phi_0$ then $s\models \phi$ according to the rules of Definition \ref{Truth0}. Furthermore,
\begin{itemize}
\item  $s\models(\phi >\psi)$ (with $\phi,\psi\in\Phi_0$) if and only if either $\Vert \phi \Vert=\varnothing$, or $\Vert \phi \Vert\neq\varnothing$ and $f(s,\Vert \phi \Vert)\subseteq\Vert\psi\Vert$,
\item $s\models\mathbb B\phi$ if and only if $\mathcal B(s)\subseteq \Vert \phi \Vert$.
\end{itemize}
\end{definition}
Then we can re-write the definition of $K\div\phi$ given in \eqref{EQ:K-} in terms of the modal operators $\mathbb B$ and $>$ as follows:
\begin{equation}
\label{EQ:K-modal}
\psi\in K\div \phi\text{ if and only if }\phi,\psi\in\Phi_0\text{ and }s_@\models\mathbb  B\psi\wedge\mathbb  B\left(\lnot\phi>\psi\right).
\end{equation}
Thus, in the statement of Proposition \ref{PROP:char}, $K=\{\phi \in \Phi _{0}:\mathcal B(s_@)\subseteq \Vert\phi\Vert\}$ can be replaced by $K=\{\phi\in\Phi_0:s_@\models\mathbb B\phi\}$  and reference to \eqref{EQ:K-} can be replaced by reference to \eqref{EQ:K-modal}.  Note that only a fragment of the extended language is used in the characterization result of Proposition \ref{PROP:char}. In particular, nesting of conditionals and beliefs is disallowed. The study of whether the extended language can be used to obtain generalizations of AGM-style belief change that go beyond merely Boolean expressions is a topic left for future research.
\section{Related literature}
There is a vast literature that deals with AGM belief contraction (for a survey see, for example, \cite{FerHans11,FerHans18}). Because of space limitations we will only focus on a few issues.
\par
The recovery postulate (AGM axiom ($K-5$)) appears to be a natural way of capturing a ``minimal'' way of suspending belief in $\phi$, but has been subject to extensive scrutiny (see \cite{Mak87,Fuh91,Hans91,Lev91,LinRab91,Nie91,Hans96,Hans99b}). In Makinson's terminology (\cite{Mak87}), contraction operations that do not satisfy the recovery postulates are called \emph{withdrawals}. Alternative types of withdrawal operators have been studied in the literature: contraction without recovery (\cite{Fer98}), semi-contraction (\cite{FerRod98}), severe withdrawal (\cite{PagRot99}), systematic withdrawal (\cite{Meyetal02}), mild contraction (\cite{Lev04}). If one interprets belief contraction as a form of \emph{actual belief change} (in response to some input), then perhaps the recovery postulate is open to scrutiny. However, in the interpretation of belief contraction proposed in this paper, the recovery postulate is entirely natural. Indeed, if $\psi$ belongs to the contraction of $K$ by $\phi$ then $\psi$ is believed both initially and  on the supposition that $\lnot\phi$; if this supposition is removed then one naturally falls back to the initial beliefs $K$.
\par
There have been attempts in the literature to establish a link between notions of AGM belief change and Stalnaker-Lewis conditionals. Within the context of AGM belief revision this was done by \cite{Gae86}, who considered the language that we called $\Phi_1$, which includes conditionals of the form $\phi>\psi$. G\"{a}erdenfors introduced the following postulate (where $K*\phi$ denotes the revised belief set in response to information $\phi$): $(\phi>\psi)\in K \text{ if and only if } \psi\in K{\Large *}\phi$. This postulate was taken to be an expression of the so-called Ransey test.\footnote{
The expression "Ramsey Test" refers to the following passage from \cite[p. 247]{Ram50}: "If two people are arguing "If $p$ will $q$?" and are both in doubt as  to $p$, they are adding $p$ hypothetically to their stock of knowledge and arguing on that  basis about $q$".
}\ G\"{a}erdenfors showed that this postulate can be satisfied only in cases where the revision operation is trivial; in other words, there cannot be interesting revision theories based on conditionals if one requires that the conditionals themselves be incorporated in the initial belief set. Several attempts have been made to circumvent G\"{a}erdenfors' ``triviality result''. Different routes have been taken: weakening or re-interpretating the theorem (\cite{LinRab91,LinRab92,LinRab98,Rot86,Rot17}, generalizing from belief revision functions to belief change systems (consisting of a set of epistemic states, an assignment of a belief set to each epistemic state and a transition function function that determines how the epistemic state changes as a result of learning new information: \cite{FriHal94}), considering an alternative semantics, namely Moss and Parikh’s epistemic logic of subsets logic (\cite{MosPar92}), and augmenting it with conditionals (\cite{Geo17}), and, in the context of iterated belief contraction, defining the notion of "contractional" in the context of belief states ( \cite{Sauetal20}: if $\Psi$ denotes a belief state and $[\beta|\alpha]$ is interpreted as ``belief in $\beta$ even in the absence of $\alpha$", then the contractional is defined as $\Psi\models[\beta|\alpha]$ if and only if $\Psi\div\alpha\models\beta$). None of the approaches described above coincides with the framework considered in this paper.
\section{Conclusion}
We proposed a semantic characterization of AGM belief contraction in terms of a semantics consisting of a Kripke belief relation $\mathcal B$ (with associated modal operator $\mathbb B$) and a Stalnaker-Lewis selection function $f$ (with associated conditional bimodal operator $>$). The proposed semantics can also be used to characterize AGM belief revision (see \cite{Bon23}). Indeed all three operations: belief expansion, belief contraction and belief revision, can be captured within this framework. Letting $s_@$ denote the actual state, we have:
\begin{enumerate}
\item Expansion: $\psi\in K+\phi$ if and only if $s_@\models \lnot\mathbb B\lnot\phi\wedge\mathbb B(\phi\rightarrow\psi)$,
\item Contraction: $\psi\in K\div\phi$ if and only if $s_@\models \mathbb B\psi\wedge\mathbb B(\lnot\phi>\psi)$,
\item Revision: $\psi\in K*\phi$ if and only if $s_@\models\mathbb B(\phi>\psi)$.
\end{enumerate}
There are several issues that can be studied within this framework and are left for future work, for example,  whether the extended modal language can provide a way to generalize AGM-style belief change and whether  the proposed framework can accommodate iterated belief contraction/revision.
\appendix
\section{Appendix}
In this Appendix we prove Proposition 1. In order to make the proof entirely self-contained we include the proofs of known auxiliary results (e.g. the lemmas).\footnote{
Which can be found, for example, in \cite{Gae88,Hans99c}.
}
\begin{lemma}
\label{LEM:KfromModel}
Fix a model $M=\left\langle {S,s_@,\mathcal B,f,V} \right\rangle$ and let $K=\{\phi \in \Phi _{0}:\mathcal B(s_@) \subseteq \Vert \phi \Vert \}$. Then $K$ is deductively closed and consistent.
\end{lemma}
\begin{proof}
First we show that $K$ is deductively closed, that is, $K=Cn(K)$. If $\psi \in K$ then $\psi \in Cn(K)$, because $\psi \rightarrow \psi $ is a tautology; thus $K\subseteq Cn(K)$. To show that $Cn(K)\subseteq K$, let $\psi \in Cn(K)$, that is, there exist $\phi _{1},...,\phi _{n}\in K$ ($n\geq 0$) such that $( \phi _{1}\wedge ...\wedge \phi _{n}) \rightarrow \psi $ is a tautology. Since $\Vert \phi _{1}\wedge ...\wedge \phi _{n}\Vert =\Vert \phi _{1}\Vert \cap ...\cap \Vert \phi _{n}\Vert $ and,  for all $i=1,...,n$, $\phi _{i}\in K$ (that is, $\mathcal B(s_@)\subseteq \Vert \phi _{i}\Vert $), it follows that $\mathcal B(s_@) \subseteq \Vert \phi _{1}\wedge ...\wedge \phi _{n}\Vert $. Since $( \phi _{1}\wedge ...\wedge \phi _{n}) \rightarrow \psi $ is a tautology, $\Vert ( \phi _{1}\wedge ...\wedge \phi _{n}) \rightarrow \psi \Vert =S $, that is, $\Vert \phi _{1}\wedge ...\wedge \phi _{n}\Vert \subseteq \Vert \psi \Vert .$ Thus $\mathcal B(s_@)\subseteq \Vert \psi \Vert $, that is, $\psi \in K.$ Next we show that $Cn(K)\neq \Phi _{0}$, that is, $K$ is consistent. Let $p\in\texttt{At}$ be an atomic formula. Then $\Vert p\wedge \lnot p\Vert =\varnothing $. By seriality of $\mathcal B$, $\mathcal B(s_@)\neq\varnothing$ so that $\mathcal B(s_@)\nsubseteq \Vert p\wedge \lnot p\Vert $, that is, $(p\wedge \lnot p)\notin K$ and hence, since $K=Cn(K)$, $(p\wedge \lnot p)\notin Cn(K).$
\end{proof}
\noindent\textbf{Proof of Part (A) of Proposition 1}.\\
Fix a model $\left\langle {S,s_@,\mathcal B,f,V} \right\rangle$ and let  $K=\{\phi \in \Phi _{0}:\mathcal B(s_@)\subseteq \Vert\phi\Vert\} $ and, for every $\phi\in\Phi_0$ such that $\Vert\lnot\phi\Vert\ne\varnothing$, let $K\div \phi$ be defined as follows (\eqref{A1} below reproduces \eqref{EQ:K-} above):
\begin{equation}
\label{A1}
\tag{A1}
\begin{array}{*{20}{l}}
 \psi\in K\div \phi\text{ if and only if }&(1)\, \mathcal B(s_@)\subseteq\Vert\psi\Vert,\text{ (that is, }\psi\in K)\text{ and } \\
{}&(2)\, \forall s\in\mathcal B(s_@), f\left(s,\Vert\lnot\phi\Vert\right)\subseteq\Vert\psi\Vert.
\end{array}
\end{equation}
Let `$\div' $' be following extension to $\Phi_0$ of the operator `$\div $' defined in \eqref{A1}:
\begin{equation}
\label{A2}
\tag{A2}
K\div' \phi=\left\{ {\begin{array}{*{20}{l}}
K\div \phi&\text{if }\Vert\lnot\phi\Vert\ne\varnothing \\
K\cap Cn(\lnot\phi)&\text{if }\Vert\lnot\phi\Vert=\varnothing.
\end{array}} \right.
\end{equation}
We want to show that the contraction operator defined in \eqref{A2} satisfies the AGM axioms.
\paragraph{($K-$1)} We need to show that, for every $\phi\in\Phi_0$, $K\div' \phi=Cn\left(K\div' \phi\right)$. If $\Vert\lnot\phi\Vert=\varnothing$ then this is true by construction, since $K$ is deductively closed and the intersection of deductively closed sets is deductively closed. Assume, therefore, that $\Vert\lnot\phi\Vert\ne\varnothing$, so that $K\div' \phi= K\div \phi$. Note first that, by \eqref{A1}, letting
\begin{equation}
\tag{A3}\label{A3}
\Psi_{\lnot\phi}=\left\{\psi\in\Phi_0 : f(s,\Vert\lnot\phi\Vert)\subseteq\Vert\psi\Vert,\forall s\in\mathcal B(s_@)\right\},
\end{equation}
$K\div \phi= K \cap \Psi_{\lnot\phi}$. Since the intersection of two deductively closed sets is deductively closed and $K$ is deductively closed, it suffices to show that $\Psi_{\lnot\phi}$ is deductively closed, that is, $\Psi_{\lnot\phi}=Cn(\Psi_{\lnot\phi})$. The inclusion $\Psi_{\lnot\phi}\subseteq Cn(\Psi_{\lnot\phi})$ follows from the fact that, for every $\chi\in \Psi_{\lnot\phi}$, $\chi\rightarrow\chi$ is a tautology. Next we show that $Cn(\Psi_{\lnot\phi})\subseteq \Psi_{\lnot\phi}$.  Since $\Vert \lnot\phi\Vert \ne \varnothing$,  $f(s,\Vert\lnot\phi\Vert)$ is defined for every $s\in \mathcal B(s_@)$. Fix an arbitrary $\psi \in Cn(\Psi_{\lnot\phi})$; then there exist $\phi _{1},...,\phi _{n}\in \Psi_{\lnot\phi}$ ($n\geq 0$) such that $( \phi _{1}\wedge ...\wedge \phi _{n}) \rightarrow \psi $ is a tautology, so that $\Vert( \phi _{1}\wedge ...\wedge \phi _{n}) \rightarrow \psi  \Vert = S$, that is, $\Vert\phi _{1}\wedge ...\wedge \phi _{n}\Vert\subseteq\Vert\psi\Vert$. Fix an arbitrary $s\in\mathcal B(s_@)$ and an arbitrary $i=1,...,n$. Then, since $\phi _{i}\in \Psi_{\lnot\phi}$, $ f(s,\Vert\lnot\phi\Vert)\subseteq\Vert\phi_i\Vert$. Hence $f(s,\Vert\lnot\phi\Vert)\subseteq\Vert\phi _{1}\wedge ...\wedge \phi _{n}\Vert$. Since $\Vert(\phi _{1}\wedge ...\wedge \phi _{n})\Vert\subseteq\Vert\psi\Vert$ it follows that $f(s,\Vert\lnot\phi\Vert)\subseteq\Vert\psi\Vert$, that is, $\psi\in\Psi_{\lnot\phi}$.
\paragraph{($K-$2)}  We need to show that $K\div' \phi\subseteq K$. If $\Vert\lnot\phi\Vert=\varnothing$ then $K\div' \phi=K\cap Cn(\lnot\phi)\subseteq K$. If $\Vert\lnot\phi\Vert\ne\varnothing$ then $K\div' \phi=K\div \phi= K \cap \Psi_{\lnot\phi}\subseteq K$.
\paragraph{($K-$3)} We need to show that if $\phi\notin K$ then $K\subseteq K\div' \phi$. Assume that $\phi\notin K$, that is, $\mathcal B(s_@)\cap \Vert\lnot\phi\Vert\ne\varnothing$. Then $\Vert\lnot\phi\Vert\ne\varnothing$ and thus $K\div' \phi=K\div \phi$. Fix an arbitrary $\psi\in K$, that is, $\mathcal B(s_@)\subseteq\Vert\psi\Vert$. We need to show that, $\forall s\in\mathcal B(s_@)$, $f(s,\Vert\lnot\phi\Vert)\subseteq\Vert\psi\Vert$. Since $\mathcal B(s_@)\cap \Vert\lnot\phi\Vert\ne\varnothing$, by 4(c) of Definition 2, for every $s\in\mathcal B(s_@)$, $f(s,\Vert\lnot\phi\Vert)\subseteq\mathcal B(s_@)\cap \Vert\lnot\phi\Vert$ and thus, since $\mathcal B(s_@)\subseteq\Vert\psi\Vert$, $f(s,\Vert\lnot\phi\Vert)\subseteq\Vert\psi\Vert$.
\paragraph{($K-$4)} We need to show that if $\phi$ is not a tautology then $\phi\notin K\div' \phi$. Suppose that $\phi$ is not a tautology, so that $\phi\notin Cn(\lnot\phi)$. If $\Vert\lnot\phi\Vert=\varnothing$ then $K\div' \phi=K\cap Cn(\lnot\phi)$ and thus $\phi\notin K\div' \phi$. Next, suppose that $\Vert\lnot\phi\Vert\ne\varnothing$ so that $K\div' \phi=K\div \phi$. Since $K\div \phi= K \cap \Psi_{\lnot\phi}$ (where $\Psi_{\lnot\phi}$ is given by \eqref{A3}) it is sufficient to show that $\phi\notin\Psi_{\lnot\phi}$, that is, $f(s,\Vert\lnot\phi\Vert)\not\subseteq\Vert\phi\Vert$, for some $s\in\mathcal B(s_@)$. This follows from the fact that, by 4(a) of Definition 2, for every $s\in\mathcal B(s_@)$, $f(s,\Vert\lnot\phi\Vert)\subseteq\Vert\lnot\phi\Vert$.
\paragraph{($K-$5)} We need to show that if $\phi\in K$ then $K\subseteq (K\div' \phi)+\phi=Cn(K\div' \phi\cup\{\phi\})$. Assume that $\phi\in K$ and fix an arbitrary $\psi\in K$. Then $(\phi\rightarrow\psi)\in K$. If $\Vert\lnot\phi\Vert=\varnothing$ then $K\div' \phi = K\cap Cn(\lnot\phi)$. Since $\lnot\phi\in Cn(\lnot\phi)$, $\phi\rightarrow\psi\in Cn(\lnot\phi)$ and thus $\phi\rightarrow\psi\in K\div' \phi$, from which it follows (since, by ($K-1$), $K\div' \phi$ is deductively closed) that $\psi\in Cn(K\div' \phi\cup\{\phi\})$. Suppose now that $\Vert\lnot\phi\Vert\ne\varnothing$ so that $K\div' \phi=K\div \phi=K \cap \Psi_{\lnot\phi}$ (where $\Psi_{\lnot\phi}$ is given by \eqref{A3}). By 4(a) of Definition 2, for every $s\in\mathcal B(s_@)$, $f(s,\Vert\lnot\phi\Vert)\subseteq\Vert\lnot\phi\Vert$ and thus $f(s,\Vert\lnot\phi\Vert)\subseteq\Vert\phi\rightarrow\psi\Vert=\Vert\lnot\phi\Vert\cup\Vert\psi\Vert$. Hence (recall that $(\phi\rightarrow\psi)\in K$) $(\phi\rightarrow\psi)\in K\div \phi$ so that $\psi\in Cn(K\div \phi\cup\{\phi\})$.
\paragraph{($K-$6)} We need to show that if $\phi\leftrightarrow\psi$ is a tautology then $K\div' \phi=K\div' \psi$. Assume that $\phi\leftrightarrow\psi$ is a tautology. Then $Cn(\lnot\phi)=Cn(\lnot\psi)$ and $\Vert\lnot\phi\Vert=\Vert\lnot\psi\Vert$. Thus $\Vert\lnot\phi\Vert=\varnothing$ if and only if $\Vert\lnot\psi\Vert=\varnothing$, in which case $K\div' \phi=K\cap Cn(\lnot\phi)=K\cap Cn(\lnot\psi)=K\div' \psi$. Furthermore, $\Vert\lnot\phi\Vert\ne\varnothing$ if and only if $\Vert\lnot\psi\Vert\ne\varnothing$,  in which case $\{\chi\in\Phi_0 : f(s,\Vert\lnot\phi\Vert)\subseteq\Vert\chi\Vert,\forall s\in\mathcal B(s_@)\}=\{\chi\in\Phi_0 : f(s,\Vert\lnot\psi\Vert)\subseteq\Vert\chi\Vert,\forall s\in\mathcal B(s_@)\}$, from which it follows that $K\div \phi=K\div \psi$.
\paragraph{($K-$7)} We have to show that $(K\div' \phi)\cap(K\div' \psi)\subseteq K\div' (\phi\wedge\psi)$. We need to consider several cases.\\[3pt]
Case 1: $\Vert\lnot\phi\Vert=\Vert\lnot\psi\Vert=\varnothing$ so that $\Vert\lnot\phi\Vert\cup\Vert\lnot\psi\Vert=\Vert\lnot\phi\vee\lnot\psi\Vert=\Vert\lnot(\phi\wedge\psi)\Vert=\varnothing$. In this case $K\div' \phi=K\cap Cn(\lnot\phi)$, $K\div' \psi=K\cap Cn(\lnot\psi)$ and $K\div' (\phi\wedge\psi)=K\cap Cn(\lnot(\phi\wedge\psi))$. Since $Cn(\lnot\phi)\cap Cn(\lnot\psi)\subseteq Cn(\lnot\phi\vee\lnot\psi)=Cn(\lnot(\phi\wedge\psi))$ it follows that $(K\div' \phi)\cap(K\div' \psi)\subseteq K\div' (\phi\wedge\psi)$.\\[3pt]
Case 2: $\Vert\lnot\phi\Vert=\varnothing$ and $\Vert\lnot\psi\Vert\ne \varnothing$, so that $\Vert\lnot(\phi\wedge\psi)\Vert=\Vert\lnot\phi\vee\lnot\psi\Vert=\Vert\lnot\phi\Vert\cup\Vert\lnot\psi\Vert=\Vert\lnot\psi\Vert\ne\varnothing$.
In this case $K\div' \phi=K\cap Cn(\lnot\phi)$, $K\div' \psi=K\div \psi=K\cap \{\chi\in\Phi_0 : f(s,\Vert\lnot\psi\Vert)\subseteq\Vert\chi\Vert,\forall s\in\mathcal B(s_@)\}$ and $K\div' (\phi\wedge\psi)=K\div (\phi\wedge\psi)=K\cap \{\chi\in\Phi_0 : f(s,\Vert\lnot(\phi\wedge\psi)\Vert)\subseteq\Vert\chi\Vert,\forall s\in\mathcal B(s_@)\}$. Since $\Vert\lnot(\phi\wedge\psi)\Vert=\Vert\lnot\psi\Vert$, $f(s,\Vert\lnot(\phi\wedge\psi)\Vert)= f(s,\Vert\lnot\psi\Vert)$ and thus $K\div (\phi\wedge\psi)=K\div \psi$. Hence the inclusion $(K\div' \phi)\cap(K\div \psi)\subseteq K\div (\phi\wedge\psi)$ reduces to $(K\div' \phi)\cap(K\div \psi)\subseteq K\div\psi$, which is trivially true.\\[3pt]
Case 3: $\Vert\lnot\phi\Vert\ne\varnothing$ and $\Vert\lnot\psi\Vert= \varnothing$, so that $\Vert\lnot\phi\vee\lnot\psi\Vert=\Vert\lnot\phi\Vert\cup\Vert\lnot\psi\Vert=\Vert\lnot\phi\Vert\ne\varnothing$.
In this case, by an argument similar to the one used in Case 2, $K\div' (\phi\wedge\psi)=K\div (\phi\wedge\psi)=K\div\phi=K\div' \phi$, so that the inclusion $(K\div' \phi)\cap(K\div' \psi)\subseteq K\div' (\phi\wedge\psi)$ reduces to $(K\div\phi)\cap(K\div'\psi)\subseteq K\div\phi$, which is trivially true.\\[3pt]
Case 4: $\Vert\lnot\phi\Vert\ne\varnothing$ and $\Vert\lnot\psi\Vert\ne\varnothing$, so that $\Vert\lnot(\phi\wedge\psi)\Vert=\Vert\lnot\phi\vee\lnot\psi\Vert=\Vert\lnot\phi\Vert\cup\Vert\lnot\psi\Vert\ne\varnothing$. In this case $K\div' \phi=K\div \phi=K\cap\{\chi\in\Phi_0 : f(s,\Vert\lnot\phi\Vert)\subseteq\Vert\chi\Vert,\forall s\in\mathcal B(s_@)\}$, $K\div' \psi=K\div \psi=K\cap\{\chi\in\Phi_0 : f(s,\Vert\lnot\psi\Vert)\subseteq\Vert\chi\Vert,\forall s\in\mathcal B(s_@)\}$ and $K\div' (\phi\wedge\psi)=K\div (\phi\wedge\psi)=K\cap\{\chi\in\Phi_0 : f(s,\Vert\lnot(\phi\wedge\psi)\Vert)\subseteq\Vert\chi\Vert,\forall s\in\mathcal B(s_@)\}$. Fix an arbitrary $\chi\in(K\div \phi)\cap(K\div \psi)$ (thus, in particular, $\chi\in K$). We need to show that $\chi\in K\div (\phi\wedge\psi)$, that is, that, $\forall s\in\mathcal B(s_@)$, $f(s,\Vert\lnot(\phi\wedge\psi)\Vert)\subseteq\Vert\chi\Vert$. Since $\chi\in (K\div \phi)\cap(K\div \psi)$,
\begin{equation}
\label{A4}
\tag{A4}
f(s,\Vert\lnot\phi\Vert)\subseteq\Vert\chi\Vert\text{ and } f(s,\Vert\lnot\psi\Vert)\subseteq\Vert\chi\Vert.
\end{equation}
By Property 4(a) of Definition 2, $f(s,\Vert\lnot(\phi\wedge\psi)\Vert)\subseteq\Vert\lnot(\phi\wedge\psi)\Vert=\Vert\lnot\phi\Vert\cup\Vert\lnot\psi\Vert$. It follows from this that
\begin{equation}
\label{A5}
\tag{A5}
f(s,\Vert\lnot(\phi\wedge\psi)\Vert)\,=\,\left(f(s,\Vert\lnot(\phi\wedge\psi)\Vert)\cap\Vert\lnot\phi\Vert\right)\,{\mathlarger {\mathlarger {\cup}}}\, \left(f(s,\Vert\lnot(\phi\wedge\psi)\Vert)\cap\Vert\lnot\psi\Vert\right).
\end{equation}
By Property 4(d) of Definition 2 (with $E=\Vert\lnot\phi\Vert\cup\Vert\lnot\psi\Vert=\Vert\lnot(\phi\wedge\psi)\Vert$ and $F=\Vert\lnot\phi\Vert$)
\begin{equation}
\label{A6}
\tag{A6}
f(s,\Vert\lnot(\phi\wedge\psi)\Vert)\cap\Vert\lnot\phi\Vert\,\,{\mathlarger{\subseteq}}\,\,f(s,\Vert\lnot\phi\Vert).
\end{equation}
A second application of Property 4(d) of Definition 2 (with $E=\Vert\lnot\phi\Vert\cup\Vert\lnot\psi\Vert=\Vert\lnot(\phi\wedge\psi)\Vert$ and, this time, with  $F=\Vert\lnot\psi\Vert$) gives
\begin{equation}
\label{A7}
\tag{A7}
f(s,\Vert\lnot(\phi\wedge\psi)\Vert)\cap\Vert\lnot\psi\Vert\,\,{\mathlarger{\subseteq}}\,\,f(s,\Vert\lnot\psi\Vert).
\end{equation}
It follows from \eqref{A5},  \eqref{A6},  \eqref{A7} that  $f(s,\Vert\lnot(\phi\wedge\psi)\Vert\subseteq (f(s,\Vert\lnot\phi\Vert)\cup f(s,\Vert\lnot\psi\Vert))$ and thus, by \eqref{A4}, $f(s,\Vert\lnot(\phi\wedge\psi)\Vert)\subseteq\Vert\chi\Vert$.
\paragraph{($K-$8)} We need to show that if $\phi\notin K\div' (\phi\wedge\psi)$ then $K\div' (\phi\wedge\psi)\subseteq K\div' \phi$. Assume that $\phi\notin K\div' (\phi\wedge\psi)$. \\
Suppose first that $\Vert\lnot\phi\Vert=\varnothing$, that is,  $\Vert\phi\Vert=S$. Then $\mathcal B(s_@)\subseteq\Vert\phi\Vert$ and thus $\phi\in K$. If $\Vert\lnot(\phi\wedge\psi)\Vert=\Vert\lnot\phi\Vert\cup\Vert\lnot\psi\Vert\ne\varnothing$ then $K\div' (\phi\wedge\psi)=K\div (\phi\wedge\psi)=K\cap\{\chi\in\Phi_0 : f(s,\Vert\lnot\phi\Vert\cup\Vert\lnot\psi\Vert)\subseteq\Vert\chi\Vert,\forall s\in\mathcal B(s_@)\}$ and, since $\Vert\phi\Vert=S$, for all $s\in\mathcal B(s_@)$ we have that $f(s,\Vert\lnot\phi\Vert\cup\Vert\lnot\psi\Vert)\subseteq\Vert\phi\Vert$, implying that $\phi\in K\div (\phi\wedge\psi)$, contradicting our assumption. Thus the case where $\Vert\lnot\phi\Vert=\varnothing$ and $\Vert\lnot\phi\Vert\cup\Vert\lnot\psi\Vert\ne\varnothing$ is ruled out and we are left with only two cases to consider.\\[3pt]
Case 1: $\Vert\lnot\phi\Vert\cup\Vert\lnot\psi\Vert=\varnothing$ so that $\Vert\lnot\phi\Vert=\varnothing$. In this case $K\div' (\phi\wedge\psi)=K\cap Cn(\lnot(\phi\wedge\psi))$ and $K\div' \phi=K\cap Cn(\lnot\phi)$. Fix an arbitrary $\chi\in K\div' (\phi\wedge\psi)$. Then $\chi\in K$ and $\chi\in Cn(\lnot(\phi\wedge\psi))$. We need to show that $\chi\in K\div' \phi$, that is, that $\chi\in Cn(\lnot\phi)$. Since $\chi\in Cn(\lnot(\phi\wedge\psi))$, $\lnot(\phi\wedge\psi)\rightarrow\chi$ is a tautology. Thus, since $\lnot\phi\rightarrow\lnot(\phi\wedge\psi)$ is also a tautology, $\lnot\phi\rightarrow\chi$ is a tautology and thus $\chi\in Cn(\lnot\phi)$.\\[3pt]
Case 2:
$\Vert\lnot\phi\Vert\ne\varnothing$ and thus $\Vert\lnot(\phi\wedge\psi)\Vert=\Vert\lnot\phi\Vert\cup\Vert\lnot\psi\Vert\ne\varnothing$. Then $K\div' (\phi\wedge\psi)=K\div (\phi\wedge\psi)=K\cap\{\chi\in\Phi_0 : f(s,\Vert\lnot\phi\Vert\cup\Vert\lnot\psi\Vert)\subseteq\Vert\chi\Vert,\forall s\in\mathcal B(s_@)\}$ and $K\div' \phi=K\div \phi=K\cap\{\chi\in\Phi_0 : f(s,\Vert\lnot\phi\Vert)\subseteq\Vert\chi\Vert,\forall s\in\mathcal B(s_@)\}$. Recall the assumption that $\phi\notin K\div (\phi\wedge\psi)$. Then two sub-cases are possible.\\
Case 2.1: $\phi\notin K$, that is, $\mathcal B(s_@)\cap\Vert\lnot\phi\Vert\ne\varnothing$. Then, by 4(c) of Definition 2,
\begin{equation}
\label{A8}
\tag{A8}
\forall s\in\mathcal B(s_@),\, f(s,\Vert\lnot\phi\Vert)\subseteq \mathcal B(s_@)\cap\Vert\lnot\phi\Vert\subseteq\mathcal B(s_@).
\end{equation}
Fix an arbitrary $\chi\in K\div (\phi\wedge\psi)$. Then $\chi\in K$, that is, $\mathcal B(s_@)\subseteq \Vert\chi\Vert$ and thus, by \eqref{A8}, $\forall s\in\mathcal B(s_@)$,  $f(s,\Vert\lnot\phi\Vert)\subseteq \Vert\chi\Vert$ so that $\chi\in K\div \phi$.\\
Case 2.2: $\phi\in K$ and $B_{\lnot\phi\lnot\psi}\ne\varnothing$, where $B_{\lnot\phi\lnot\psi}=\{s\in\mathcal B(s_@): f(s,\Vert\lnot\phi\Vert\cup\Vert\lnot\psi\Vert)\cap\Vert\lnot\phi\Vert\ne\varnothing\}$.\footnote{
Note that the case where $\phi\in K$ and  $B_{\lnot\phi\lnot\psi}=\varnothing$ is ruled out by our initial assumption that $\phi\notin  K\div (\phi\wedge\psi)$. In fact, $B_{\lnot\phi\lnot\psi}=\varnothing$ means that, $\forall s\in\mathcal B(s_@), f(s,\Vert\lnot\phi\Vert\cup\Vert\lnot\psi\Vert)\cap\Vert\lnot\phi\Vert=\varnothing$, that is, $f(s,\Vert\lnot\phi\Vert\cup\Vert\lnot\psi\Vert)\subseteq\Vert\phi\Vert$, which, in conjunction with the hypothesis that $\phi\in K$, yields $\phi\in K\div (\phi\wedge\psi)$.  
}\
 Then, by 4(e.1) of Definition 2 (with $E=\Vert\lnot\phi\Vert\cup\Vert\lnot\psi\Vert$ and $F=\Vert\lnot\phi\Vert$)
\begin{equation}
\label{A9}
\tag{A9}
\forall s\in\mathcal B_{\lnot\phi\lnot\psi}, f(s,\Vert\lnot\phi\Vert)\subseteq f(s,\Vert\lnot\phi\Vert\cup\Vert\lnot\psi\Vert)\,\cap\,\Vert\lnot\phi\Vert
\end{equation}
and, by 4(e.2) of Definition 2 (again, with $E=\Vert\lnot\phi\Vert\cup\Vert\lnot\psi\Vert$ and $F=\Vert\lnot\phi\Vert$),
\begin{equation}
\label{A10}
\tag{A10}
\forall s\in\mathcal B(s_@), f(s,\Vert\lnot\phi\Vert)\subseteq f(s',\Vert\lnot\phi\Vert)\text{ for some } s'\in B_{\lnot\phi\lnot\psi}.
\end{equation}
Fix an arbitrary $\chi\in K\div (\phi\wedge\psi)$. Then, $\chi\in K$ and (recall that $\Vert\lnot(\phi\wedge\psi)\Vert=\Vert\lnot\phi\Vert\cup\Vert\lnot\psi\Vert$)   $f(s,\Vert\lnot\phi\Vert\cup\Vert\lnot\psi\Vert)\subseteq\Vert\chi\Vert$, $\forall s\in \mathcal B(s_@)$; it follows from this, \eqref{A9} and \eqref{A10} that, $\forall s\in \mathcal B(s_@)$, $f(s,\Vert\lnot\phi\Vert)\subseteq\Vert\chi\Vert$. Thus $\chi\in K\div \phi$.
\bigskip\par
Before we proceed to the proof of Part (B) of Proposition 1, we establish the following lemma.
\begin{lemma}
\label{LEM:lem1}
Let $A\subseteq\Phi_0$ be such that $A=Cn(A)$. Then, $\forall \alpha\in\Phi_0$, $\Vert Cn\left(A\cup\{\alpha\}\right)\Vert=\Vert A\Vert\cap\Vert\alpha\Vert.$
\end{lemma}
\begin{proof}
Since $A$ is deductively closed,  $\forall \beta\in\Phi_0$,
\begin{equation}
\label{EQ:lem1}
\tag{A12}
\beta\in Cn\left(A\cup\{\alpha\}\right)\text{ if and only if } (\alpha\rightarrow\beta)\in A.
\end{equation}
First we show that $\Vert A\Vert\cap\Vert\alpha\Vert\subseteq\Vert Cn\left(A\cup\{\alpha\}\right)\Vert.$
Fix an arbitrary $s\in\Vert A\Vert\cap\Vert\alpha\Vert$; we need to show that  $s\in\Vert Cn\left(A\cup\{\alpha\}\right)\Vert$, that is, that $\forall\beta\in Cn\left(A\cup\{\alpha\}\right),\, \beta\in s$. Since $s\in\Vert\alpha\Vert$, $\alpha\in s$. Fix an arbitrary $\beta\in Cn\left(A\cup\{\alpha\}\right)$; then, by \eqref{EQ:lem1}, $(\alpha\rightarrow\beta)\in A$; thus, since $s\in\Vert A\Vert$, $(\alpha\rightarrow\beta)\in s$. Hence, since both $\alpha$ and $\alpha\rightarrow\beta$ belong to $s$ and $s$ is deductively closed, $\beta\in s$.\\
Next we show that $\Vert Cn\left(A\cup\{\alpha\}\right)\Vert\subseteq\Vert A\Vert\cap\Vert\alpha\Vert.$ Let $s\in \Vert Cn\left(A\cup\{\alpha\}\right)\Vert$. Then, since $\alpha\in Cn\left(A\cup\{\alpha\}\right)$, $\alpha\in s$, that is, $s\in\Vert\alpha\Vert$. It remains to show that $s\in\Vert A\Vert$, that is, that, for every $\beta\in A$, $\beta\in s$. Fix an arbitrary  $\beta\in A$; then, since $A$ is deductively closed, $(\alpha\rightarrow\beta)\in A$. Thus, by \eqref{EQ:lem1}, $\beta\in Cn\left(A\cup\{\alpha\}\right)$ and thus, since $s\in\Vert Cn\left(A\cup\{\alpha\}\right)\Vert$, $\beta\in s$.
\end{proof}
\noindent
\textbf{Proof of Part (B) of Proposition 1}.\\
We need to show that if $K\subset\Phi_0$ is consistent and deductively closed and $\div :\Phi_0\rightarrow 2^{\Phi_0}$ is an AGM belief contraction function based on $K$, then there exists a model $\left\langle {S,s_@,\mathcal B,f,V} \right\rangle$ such that $K=\{\phi\in\Phi_0:\mathcal B(s_@)\subseteq\Vert\phi\Vert\}$ and, for all $\phi,\psi\in\Phi_0$, $\psi\in K\div \phi$ if and only if \eqref{A1} is satisfied. Define the following model $\left\langle {S,s_@,\mathcal B,f,V} \right\rangle$:
\begin{enumerate}
  \item $S$ is the set of maximally consistent sets of formulas in $\Phi_0$.
  \item The valuation $V:\textit{At}\rightarrow S$ is defined by $V(p)=\{s\in S: p\in s\}$, so that, for every $\phi\in\Phi_0$, $\Vert\phi\Vert=\{s\in S: \phi\in s\}$.  If $\Psi\subseteq\Phi_0$, define $\Vert\Psi\Vert=\{s\in S: \forall \phi\in \Psi, \phi\in s\}$.
  \item Choose an arbitrary $s_@\in S$ and define $\mathcal B(s_@)=\Vert K\Vert$.
 \item Let $\mathcal E = \{E\subseteq S:\varnothing\ne E=\Vert\phi\Vert \text{ for some  } \phi\in\Phi_0\}$. Define $f:\mathcal B(s_@)\times \mathcal E\rightarrow 2^S$  as follows:
\begin{equation}
\label{A11}
\tag{A11}
      \forall s\in\mathcal B(s_@),\,\, f(s,\Vert\phi\Vert) =\Vert K\div \lnot\phi\Vert\,\cap\,\Vert\phi\Vert .
\end{equation}
\end{enumerate}
\begin{remark}
\label{tautology}
If $\phi$ is a tautology then $\lnot\phi$ is a contradiction and thus (since, by hypothesis, $K$ is consistent) $\lnot\phi\notin K$. It follows from ($K-2$) and ($K-3$) that $K\div \lnot\phi=K$. Furthermore, since $\phi$ is a tautology and $K$ is deductively closed, $\phi\in K$, that is $\Vert K\Vert\subseteq\Vert\phi\Vert$ so that $\Vert K\Vert\cap \Vert \phi\Vert=\Vert K\Vert$. Hence, by \eqref{A11},  $\forall s\in\mathcal B(s_@),\,f(s,\Vert\phi\Vert) =\Vert K\Vert$. On the other hand, if $\lnot\phi$ is a tautology then $\Vert\phi\Vert=\varnothing$ and thus $\Vert\phi\Vert\notin\mathcal E$, that is, $\Vert\phi\Vert$ is not in the domain of $f$.
\end{remark}
\par
First we show that the selection function defined in \eqref{A11} satisfies Properties 4(a)-4(e) of Definition 2. In view of Remark \ref{tautology}, we can restrict attention to contingent formulas, that is, to formulas $\phi$ such that neither $\phi$ nor $\lnot\phi$ is a tautology. Denote by $\Phi_{cont}\subseteq\Phi_0$ the set of contingent formulas.\\
Recall that $S$ is the set of maximally consistent sets of formulas in $\Phi_0$ and, for $A\subseteq\Phi_0$, $\Vert A\Vert=\{s\in S: \chi\in s,\,\forall\chi\in A\}$.

\paragraph{Property 4(a)} We need to show that if $\phi\in\Phi_{cont}$ then $\Vert K\div \lnot\phi\Vert\,\cap\,\Vert\phi\Vert\subseteq\Vert\phi\Vert$, which is obviously true, and $\Vert K\div \lnot\phi\Vert\,\cap\,\Vert\phi\Vert\ne\varnothing$. Since $\phi\in\Phi_{cont}$, $\Vert\phi\Vert\ne\varnothing$ and, by ($K-4$), $\lnot\phi\notin K\div \lnot\phi$. By ($K-1$) $K\div \lnot\phi=Cn(K\div \lnot\phi)$ and thus $\lnot\phi\notin Cn(K\div \lnot\phi)$, that is, $K\div \lnot\phi$ is consistent and hence $\Vert K\div \lnot\phi\Vert\ne\varnothing$.
\paragraph{Property 4(b)}  Fix an arbitrary $s\in\mathcal B(s_@)$ and an arbitrary $\phi\in\Phi_{cont}$. We need to show that if $s\in\Vert\phi\Vert$ then $s\in f(s,\Vert\phi\Vert)=\Vert K\div\lnot\phi\Vert\cap\Vert\phi\Vert$. By construction, $\mathcal B(s_@)=\Vert K\Vert$; thus, $s\in\Vert K\Vert$. By  ($K-2$), $K\div\lnot\phi\subseteq K$ so that $\Vert K\Vert\subseteq\Vert K\div\lnot\phi\Vert$. Hence $s\in\Vert K\div\lnot\phi\Vert$. Thus if $s\in\Vert\phi\Vert$ then $s\in\Vert K\div\lnot\phi\Vert\cap\Vert\phi\Vert$.
\paragraph{Property 4(c)} We need to show that if $\mathcal B(s_@)\cap\Vert\phi\Vert\ne\varnothing$ then (since $\mathcal B(s_@)=\Vert K\Vert$ and, $\forall s\in\mathcal B(s_@)$, $f(s,\Vert\phi\Vert) =\Vert K\div \lnot\phi\Vert\,\cap\,\Vert\phi\Vert$)\, $\Vert K\div \lnot\phi\Vert\,\cap\,\Vert\phi\Vert\subseteq\Vert K\Vert\cap\Vert\phi\Vert$. If $\Vert K\Vert\cap\Vert\phi\Vert\ne\varnothing$ then $\lnot\phi\notin K$ and thus, by ($K-3$), $K\subseteq K\div\lnot\phi$, so that $\Vert K\div\lnot\phi\Vert\subseteq \Vert K\Vert$ and thus $\Vert K\div \lnot\phi\Vert\,\cap\,\Vert\phi\Vert \subseteq\Vert K\Vert\cap\Vert\phi\Vert$.
\paragraph{Property 4(d)} We need to show that if $\phi\in\Phi_{cont}$ and $\psi\in\Phi_0$, then $\forall s\in\mathcal B(s_@)$, $f(s,\Vert\phi\Vert)\,\cap\,\Vert\psi\Vert\subseteq f(s,\Vert\phi\Vert\cap\Vert\psi\Vert)$, that is, using \eqref{A11} and the fact that $\Vert\phi\Vert\cap\Vert\psi\Vert=\Vert\phi\wedge\psi\Vert$,
\begin{equation}
\label{A13}
\tag{A13}
      \Vert K\div \lnot\phi\Vert\,\cap\,\Vert\phi\Vert\,\cap\,\Vert\psi\Vert\,\subseteq\, \Vert K\div\lnot(\phi\wedge\psi)\Vert\,\cap\,\Vert\phi\wedge\psi\Vert
\end{equation}
By ($K-7$), $\forall \alpha,\beta\in\Phi_0,\,(K\div\alpha)\cap(K\div\beta)\subseteq K\div(\alpha\wedge\beta)$. Thus applying  ($K-7$) to $\alpha=\lnot(\phi\wedge\psi)$ and $\beta=\phi\rightarrow\psi$ we get
\begin{equation}
\label{A14}
\tag{A14}
K\div\lnot(\phi\wedge\psi)\,\cap\,K\div(\phi\rightarrow\psi)\,\subseteq\,K\div\left(\lnot(\phi\wedge\psi)\wedge(\phi\rightarrow\psi)\right)
\end{equation}
\noindent
Since $\lnot(\phi\wedge\psi)\wedge(\phi\rightarrow\psi)$ is logically equivalent to $\lnot\phi$, by ($K-6$) $K\div\left(\lnot(\phi\wedge\psi)\wedge(\phi\rightarrow\psi)\right)=K\div\lnot\phi$. Thus, by \eqref{A14}
\begin{equation}
\label{A15}
\tag{A15}
K\div\lnot(\phi\wedge\psi)\,\cap\,K\div(\phi\rightarrow\psi)\,\subseteq\,K\div\lnot\phi.
\end{equation}
Next we show that
\begin{equation}
\label{A16}
\tag{A16}
Cn\left(K\div\lnot(\phi\wedge\psi)\,\cup\,\{\phi\wedge\psi\}\right)\,\subseteq\,Cn\left(K\div\lnot\phi\,\cup\,\{\phi\wedge\psi\}\right).
\end{equation}
Fix an arbitrary $\chi\in Cn\left(K\div\lnot(\phi\wedge\psi)\,\cup\,\{\phi\wedge\psi\}\right)$. Then, since, by ($K-1$), $K\div\lnot(\phi\wedge\psi)$ is deductively closed,
\begin{equation}
\label{A17}
\tag{A17}
\left((\phi\wedge\psi)\rightarrow\chi\right)\in K\div\lnot(\phi\wedge\psi).
\end{equation}
\noindent
By ($K-2$), $K\div\lnot(\phi\wedge\psi)\subseteq K$ and thus, by \eqref{A17},
\begin{equation}
\label{A18}
\tag{A18}
\left((\phi\wedge\psi)\rightarrow\chi\right)\in K.
\end{equation}
\noindent
Next we show that
\begin{equation}
\label{A19}
\tag{A19}
\left((\phi\wedge\psi)\rightarrow\chi\right)\in K\div(\phi\rightarrow\psi).
\end{equation}
If $(\phi\rightarrow\psi)\notin K$ then, by ($K-3)$, $K\subseteq K\div(\phi\rightarrow\psi)$ and thus \eqref{A19} follows from  \eqref{A18}. If $(\phi\rightarrow\psi)\in K$ then, by ($K-5)$, $K\subseteq Cn(K\div(\phi\rightarrow\psi)\cup\{\phi\rightarrow\psi\})$ so that, by \eqref{A18}, $\left((\phi\wedge\psi)\rightarrow\chi\right)\in Cn\left(K\div(\phi\rightarrow\psi)\cup\{\phi\rightarrow\psi\}\right)$, that is (since, by ($K-1$), $K\div(\phi\rightarrow\psi)$ s deductively closed) $(\phi\rightarrow\psi)\rightarrow\left((\phi\wedge\psi)\rightarrow\chi\right)\in K\div(\phi\rightarrow\psi)$. Since $(\phi\rightarrow\psi)\rightarrow\left((\phi\wedge\psi)\rightarrow\chi\right)$ is logically equivalent to $\left((\phi\rightarrow\psi)\wedge(\phi\wedge\psi)\right)\rightarrow\chi$, which, in turn is logically equivalent to $(\phi\wedge\psi))\rightarrow\chi$, \eqref{A19} is satisfied. It follows from \eqref{A18}, \eqref{A19} and \eqref{A15} that $\big((\phi\wedge\psi)\rightarrow\chi\big)\in K\div\lnot\phi$, that is, that $\chi\in Cn\big(K\div\lnot\phi\,\cup\,\{\phi\wedge\psi\}\big)$, thus establishing \eqref{A16}. From \eqref{A16} we get that
\begin{equation}
\label{A20}
\tag{A20}
\Vert Cn\left(K\div\lnot\phi\,\cup\,\{\phi\wedge\psi\}\right)\Vert\,\subseteq\, \Vert Cn\left(K\div\lnot(\phi\wedge\psi)\,\cup\,\{\phi\wedge\psi\}\right)\Vert
\end{equation}
\noindent
By Lemma \ref{LEM:lem1} (with $A=K\div\lnot\phi$ and $\alpha=\phi\wedge\psi$), $\Vert Cn\left(K\div\lnot\phi\,\cup\,\{\phi\wedge\psi\}\right)\Vert$ $=\Vert K\div\lnot\phi\Vert\,\cap\,\Vert\phi\wedge\psi\Vert$  which in turn (since =$\Vert\phi\wedge\psi\Vert=\Vert\phi\Vert\cap\Vert\psi\Vert$) is equal to $\Vert K\div\lnot\phi\Vert\cap\Vert\phi\Vert \cap\Vert\psi\Vert$.  By Lemma \ref{LEM:lem1} again (with $A=K\div\lnot(\phi\wedge\psi)$ and $\alpha=\phi\wedge\psi$),  $ \Vert Cn(K\div\lnot(\phi\wedge\psi)\,\cup\,\{\phi\wedge\psi\})\Vert=\Vert K\div\lnot(\phi\wedge\psi)\Vert\,\cap\,\Vert\psi\wedge\psi\Vert$. Hence \eqref{A13} follows from \eqref{A20}.
\paragraph{Property 4(e)} Since, by \eqref{A11}, $\forall s,s'\in\mathcal B(s_@)$, $f(s,\Vert\phi\Vert)=f(s',\Vert\phi\Vert)=\Vert K\div \lnot\phi\Vert\,\cap\,\Vert\phi\Vert$, it is sufficient to show that if $\Vert K\div \lnot\phi\Vert\,\cap\,\Vert\phi\Vert\,\cap\,\Vert\psi\Vert\ne\varnothing$ then $\Vert K\div\lnot(\phi\wedge\psi)\Vert\,\cap\,\Vert\phi\wedge\psi\Vert \,\subseteq\,\Vert K\div \lnot\phi\Vert\,\cap\,\Vert\phi\Vert\,\cap\,\Vert\psi\Vert$.
Assume that $\Vert K\div \lnot\phi\Vert\,\cap\,\Vert\phi\Vert\,\cap\,\Vert\psi\Vert=\Vert K\div \lnot\phi\Vert\,\cap\,\Vert\phi\wedge\psi\Vert\ne\varnothing$. Then
\begin{equation}
\label{A21}
\tag{A21}
\lnot(\phi\wedge\psi)\notin\, K\div\lnot\phi.
\end{equation}
\noindent
Since $\lnot\phi$ is logically equivalent to $\lnot(\phi\wedge\psi)\wedge\lnot\phi$, by ($K-6$)
\begin{equation}
\label{A22}
\tag{A22}
K\div\lnot\phi = K\div\left(\lnot(\phi\wedge\psi)\wedge\lnot\phi\right).
\end{equation}
\noindent
Thus, by \eqref{A21} and \eqref{A22},
\begin{equation}
\label{A23}
\tag{A23}
\lnot(\phi\wedge\psi)\notin\, K\div\left(\lnot(\phi\wedge\psi)\wedge\lnot\phi\right).
\end{equation}
\noindent
By ($K-8$), $\forall \alpha,\beta\in\Phi_0$, if $\alpha\notin K\div(\alpha\wedge\beta)$ then $K\div(\alpha\wedge\beta)\subseteq K\div\alpha$. Thus, by \eqref{A23} and ($K-8$) (with $\alpha=\lnot(\phi\wedge\psi)$ and $\beta=\lnot\phi$), $K\div\left(\lnot\phi\wedge\lnot(\phi\wedge\psi)\right)\subseteq K\div\lnot(\phi\wedge\psi)$. It follows from this and \eqref{A22} that $K\div\lnot\phi\subseteq K\div\lnot(\phi\wedge\psi)$ and thus
\begin{equation}
\label{A24}
\tag{A24}
\Vert K\div\lnot(\phi\wedge\psi)\Vert\subseteq\Vert K\div\lnot\phi\Vert.
\end{equation}
\noindent
Intersecting both sides of \eqref{A24} with $\Vert\phi\wedge\psi\Vert=\Vert\phi\Vert\cap\Vert\psi\Vert$ we get
$\Vert K\div\lnot(\phi\wedge\psi)\Vert\cap\Vert\phi\wedge\psi\Vert\,\subseteq\,\Vert K\div\lnot\phi\Vert\cap\Vert\phi\Vert\cap\Vert\psi\Vert$, as desired.
\medskip\par
To complete the proof of Part (B) of Proposition 1 we need to show that
\begin{equation*}
\begin{array}{*{20}{l}}
 \psi\in K\div \phi\text{ if and only if }&(1)\, \mathcal B(s_@)\subseteq\Vert\psi\Vert,\text{ and } \\
{}&(2)\, \forall s\in\mathcal B(s_@), f\left(s,\Vert\lnot\phi\Vert\right)\subseteq\Vert\psi\Vert.
\end{array}
\end{equation*}
By \eqref{A11}, $\forall s\in\mathcal B(s_@)=\Vert K\Vert$, $f\left(s,\Vert\lnot\phi\Vert\right)=\Vert K\div\phi\Vert\cap\Vert\lnot\phi\Vert$. Thus we have to show that
\begin{equation}
\label{A25}
\tag{A25}
\psi\in K\div\phi\text{ if and only if } \Vert K\Vert\subseteq\Vert\psi\Vert\text{ and } \Vert K\div\phi\Vert\cap\Vert\lnot\phi\Vert\subseteq\Vert\psi\Vert.
\end{equation}
First we establish a lemma.
\begin{lemma}
\label{LEM:lem3}
$\forall\phi\in\Phi_0$,
\begin{enumerate}
  \item[($i$)] if $A\subseteq\Phi_0$ is such that $A=Cn(A)$, then $A=Cn\left(A\cup\{\phi\}\right)\,\cap\,Cn(A\cup\{\lnot\phi\})$
  \item[($ii$)] $K\div\phi=K\cap Cn(K\div\phi\cup\{\lnot\phi\})$
\end{enumerate}
\end{lemma}
\begin{proof}
($i$) Let $A\subseteq\Phi_0$ be such that $A=Cn(A)$. Since $A\subseteq Cn\left(A\cup\{\phi\}\right)$ and $A\subseteq Cn\left(A\cup\{\lnot\phi\}\right)$, $A\subseteq Cn\left(A\cup\{\phi\}\right)\,\cap\,Cn\left(A\cup\{\lnot\phi\}\right)$. Conversely, suppose that $\chi\in Cn\left(A\cup\{\phi\}\right)\,\cap\,Cn\left(A\cup\{\lnot\phi\}\right)$. Then both $\phi\rightarrow\chi$ and $\lnot\phi\rightarrow\chi$ belong to $A$ and thus so does their conjunction. Since $(\phi\rightarrow\chi)\wedge(\lnot\phi\rightarrow\chi)$ is logically equivalent to $\chi$ it follows that $\chi\in A$.\\
($ii$) We need to consider two cases. \\
Case 1: $\phi\in K$. Then, by ($K-5$), $K\subseteq Cn\left(K\div\phi\cup\{\phi\}\right)$. By ($K-2$), $K\div\phi\subseteq K$, so that $Cn\big(K\div\phi\cup\{\phi\}\big)\subseteq Cn\big(K\cup\{\phi\}\big)=Cn(K) = K$ (by hypothesis, $K$ is deductively closed). Thus
\begin{equation}
\label{A26}
\tag{A26}
K= Cn\left(K\div\phi\cup\{\phi\}\right)
\end{equation}
\noindent
By Part ($i$) (with $A=K\div\phi$, which, by ($K-1$), is deductively closed),
\begin{equation}
\label{A27}
\tag{A27}
K\div\phi= Cn\left(K\div\phi\cup\{\phi\}\right)\,\cap\,Cn\left(K\div\phi\cup\{\lnot\phi\}\right)
\end{equation}
Thus, by \eqref{A26} and \eqref{A27}, $K\div\phi=K\cap Cn(K\div\phi\cup\{\lnot\phi\})$.\\
Case 2: $\phi\notin K$. Then, by ($K-2$) and ($K-3$),
\begin{equation}
\label{A28}
\tag{A28}
K\div\phi=K
\end{equation}
By Part ($i$) (with $A=K$)
\begin{equation}
\label{A29}
\tag{A29}
K=Cn\left(K\cup\{\phi\}\right)\cap Cn\left(K\cup\{\lnot\phi\}\right)
\end{equation}
From \eqref{A29} we get that $K\cap Cn\left(K\cup\{\lnot\phi\}\right)=Cn\left(K\cup\{\phi\}\right)\cap Cn\left(K\cup\{\lnot\phi\}\right)=K$. Thus, by \eqref{A28}, $K\div\phi=K\cap Cn\left(K\cup\{\lnot\phi\}\right)$, from which, by using \eqref{A28} again to replace the second instance of $K$ with $K\div\phi$, we get $K\div\phi=K\cap Cn\left(K\div\phi\cup\{\lnot\phi\}\right)$
\end{proof}

Now we are ready to prove \eqref{A25}, namely that
\begin{equation*}
\psi\in K\div\phi\text{ if and only if } \Vert K\Vert\subseteq\Vert\psi\Vert,\text{ and } \Vert Cn\left(K\div\phi\cup\{\lnot\phi\}\right)\Vert\subseteq\Vert\psi\Vert.
\end{equation*}
Let $\psi\in K\div\phi$. By ($ii$) of Lemma \ref{LEM:lem3}, $K\div\phi=K\cap Cn\left(K\div\phi\cup\{\lnot\phi\}\right)$; thus $\psi\in K$, that is, $\Vert K\Vert\subseteq\Vert\psi\Vert$, and $\psi\in Cn\left(K\div\phi\cup\{\lnot\phi\}\right)$, that is, $\Vert Cn(K\cup\{\lnot\phi\})\Vert\subseteq\Vert\psi\Vert$. Conversely, suppose that $ \Vert K\Vert\subseteq\Vert\psi\Vert$ and $ \Vert Cn\left(K\div\phi\cup\{\lnot\phi\}\right)\Vert\subseteq\Vert\psi\Vert$, that is, $\psi\in K\,\cap\, Cn\left(K\div\phi\cup\{\lnot\phi\}\right)$. Then, by ($ii$) of Lemma \ref{LEM:lem3}, $\psi\in K\div\phi$. \quad $\square$

\nocite{*}
\bibliographystyle{eptcs}
\bibliography{Contraction}

\begin{thebibliography}{10}
\providecommand{\bibitemdeclare}[2]{}
\providecommand{\surnamestart}{}
\providecommand{\surnameend}{}
\providecommand{\urlprefix}{Available at }
\providecommand{\url}[1]{\texttt{#1}}
\providecommand{\href}[2]{\texttt{#2}}
\providecommand{\urlalt}[2]{\href{#1}{#2}}
\providecommand{\doi}[1]{doi:\urlalt{https://doi.org/#1}{#1}}
\providecommand{\eprint}[1]{arXiv:\urlalt{https://arxiv.org/abs/#1}{#1}}
\providecommand{\bibinfo}[2]{#2}

\bibitemdeclare{article}{AGM85}
\bibitem{AGM85}
\bibinfo{author}{Carlos \surnamestart Alchourr\'{o}n\surnameend},
  \bibinfo{author}{Peter \surnamestart G\"{a}rdenfors\surnameend} \&
  \bibinfo{author}{David \surnamestart Makinson\surnameend}
  (\bibinfo{year}{1985}): \emph{\bibinfo{title}{On the logic of theory change:
  partial meet contraction and revision functions}}.
\newblock {\slshape \bibinfo{journal}{The Journal of Symbolic Logic}}
  \bibinfo{volume}{50}, pp. \bibinfo{pages}{510--530}, \doi{10.2307/2274239}.

\bibitemdeclare{techreport}{Bon23}
\bibitem{Bon23}
\bibinfo{author}{Giacomo \surnamestart Bonanno\surnameend}
  (\bibinfo{year}{2023}): \emph{\bibinfo{title}{A Kripke-Stalnaker-Lewis
  semantics for AGM belief revision}}.
\newblock \bibinfo{type}{Technical Report}, \bibinfo{institution}{REPEC
  preprint No. 354}.
\newblock \urlprefix\url{https://econpapers.repec.org/paper/cdawpaper/354.htm}.

\bibitemdeclare{article}{Choetal08}
\bibitem{Choetal08}
\bibinfo{author}{Samir \surnamestart Chopra\surnameend},
  \bibinfo{author}{Aditya \surnamestart Ghose\surnameend},
  \bibinfo{author}{Thomas \surnamestart Meyer\surnameend} \&
  \bibinfo{author}{Ka-Shu \surnamestart Wong\surnameend}
  (\bibinfo{year}{2008}): \emph{\bibinfo{title}{Iterated Belief Change and the
  Recovery Axiom}}.
\newblock {\slshape \bibinfo{journal}{Journal of Philosophical Logic}}
  \bibinfo{volume}{37}(\bibinfo{number}{5}), pp. \bibinfo{pages}{501--520},
  \doi{10.1007/s10992-008-9086-2}.

\bibitemdeclare{article}{Fer98}
\bibitem{Fer98}
\bibinfo{author}{Eduardo \surnamestart Ferm\'e\surnameend}
  (\bibinfo{year}{1998}): \emph{\bibinfo{title}{On the Logic of Theory Change:
  Contraction without Recovery}}.
\newblock {\slshape \bibinfo{journal}{Journal of Logic, Language, and
  Information}} \bibinfo{volume}{7}(\bibinfo{number}{2}), pp.
  \bibinfo{pages}{127--137}, \doi{10.1023/A:1008241816078}.

\bibitemdeclare{article}{FerHans11}
\bibitem{FerHans11}
\bibinfo{author}{Eduardo \surnamestart Ferm\'e\surnameend} \&
  \bibinfo{author}{Sven~Ove \surnamestart Hansson\surnameend}
  (\bibinfo{year}{2011}): \emph{\bibinfo{title}{{AGM} 25 Years}}.
\newblock {\slshape \bibinfo{journal}{Journal of Philosophical Logic}}
  \bibinfo{volume}{40}, pp. \bibinfo{pages}{295--331},
  \doi{10.1007/S10992-011-9171-9}.

\bibitemdeclare{book}{FerHans18}
\bibitem{FerHans18}
\bibinfo{author}{Eduardo \surnamestart Ferm\'e\surnameend} \&
  \bibinfo{author}{Sven~Ove \surnamestart Hansson\surnameend}
  (\bibinfo{year}{2018}): \emph{\bibinfo{title}{Belief change: introduction and
  overview}}.
\newblock \bibinfo{publisher}{Springer}, \doi{10.1007/978-3-319-60535-7}.

\bibitemdeclare{article}{FerRod98}
\bibitem{FerRod98}
\bibinfo{author}{Eduardo \surnamestart Ferm\'e\surnameend} \&
  \bibinfo{author}{Ricardo \surnamestart Rodriguez\surnameend}
  (\bibinfo{year}{1998}): \emph{\bibinfo{title}{Semi-Contraction: Axioms and
  Construction}}.
\newblock {\slshape \bibinfo{journal}{Notre Dame Journal of Formal Logic}}
  \bibinfo{volume}{39}(\bibinfo{number}{3}), pp. \bibinfo{pages}{332--345},
  \doi{10.1305/ndjfl/1039182250}.

\bibitemdeclare{inproceedings}{FriHal94}
\bibitem{FriHal94}
\bibinfo{author}{Nir \surnamestart Friedman\surnameend} \&
  \bibinfo{author}{Joseph \surnamestart Halpern\surnameend}
  (\bibinfo{year}{1994}): \emph{\bibinfo{title}{Conditional logics of belief
  change}}.
\newblock In \bibinfo{editor}{Barbara \surnamestart Hayes-Roth\surnameend} \&
  \bibinfo{editor}{Richard \surnamestart Korf\surnameend}, editors: {\slshape
  \bibinfo{booktitle}{AAAI'94: Proceedings}}, \bibinfo{publisher}{AAAI Press},
  pp. \bibinfo{pages}{915--921}.
\newblock \urlprefix\url{https://dl.acm.org/doi/proceedings/10.5555/2891730}.

\bibitemdeclare{article}{Fuh91}
\bibitem{Fuh91}
\bibinfo{author}{Andr\'{e} \surnamestart Fuhrmann\surnameend}
  (\bibinfo{year}{1991}): \emph{\bibinfo{title}{Theory contraction through base
  contraction}}.
\newblock {\slshape \bibinfo{journal}{Journal of Philosophical Logic}}
  \bibinfo{volume}{20}, pp. \bibinfo{pages}{175--203},
  \doi{10.1007/BF00284974}.

\bibitemdeclare{article}{Gae86}
\bibitem{Gae86}
\bibinfo{author}{Peter \surnamestart G\"{a}rdenfors\surnameend}
  (\bibinfo{year}{1986}): \emph{\bibinfo{title}{Belief Revisions and the
  {R}amsey Test for Conditionals}}.
\newblock {\slshape \bibinfo{journal}{Philosophical Review}}
  \bibinfo{volume}{95}(\bibinfo{number}{1}), pp. \bibinfo{pages}{81--93},
  \doi{10.2307/2185133}.

\bibitemdeclare{book}{Gae88}
\bibitem{Gae88}
\bibinfo{author}{Peter \surnamestart G\"{a}rdenfors\surnameend}
  (\bibinfo{year}{1988}): \emph{\bibinfo{title}{Knowledge in flux: modeling the
  dynamics of epistemic states}}.
\newblock \bibinfo{publisher}{MIT Press}.
\newblock
  \urlprefix\url{https://www.collegepublications.co.uk/logic/lcs/?00004}.

\bibitemdeclare{incollection}{Geo17}
\bibitem{Geo17}
\bibinfo{author}{Konstantinos \surnamestart Georgatos\surnameend}
  (\bibinfo{year}{2017}): \emph{\bibinfo{title}{Epistemic Conditionals and the
  Logic of Subsets}}.
\newblock In \bibinfo{editor}{Ramaswamy \surnamestart Ramanujam\surnameend},
  \bibinfo{editor}{Lawrence \surnamestart Moss\surnameend} \&
  \bibinfo{editor}{Can \surnamestart Ba\c{s}kent\surnameend}, editors:
  {\slshape \bibinfo{booktitle}{Rohit Parikh on Logic, Language and Society}},
  \bibinfo{publisher}{Springer Verlag}, pp. \bibinfo{pages}{259--277},
  \doi{10.1007/978-3-319-47843-2}.

\bibitemdeclare{article}{Gro88}
\bibitem{Gro88}
\bibinfo{author}{Adam \surnamestart Grove\surnameend} (\bibinfo{year}{1988}):
  \emph{\bibinfo{title}{Two modellings for theory change}}.
\newblock {\slshape \bibinfo{journal}{Journal of Philosophical Logic}}
  \bibinfo{volume}{17}, pp. \bibinfo{pages}{157--170},
  \doi{10.1007/BF00247909}.

\bibitemdeclare{article}{Hans91}
\bibitem{Hans91}
\bibinfo{author}{Sven~Ove \surnamestart Hansson\surnameend}
  (\bibinfo{year}{1991}): \emph{\bibinfo{title}{Belief contraction without
  recovery}}.
\newblock {\slshape \bibinfo{journal}{Studia Logica}}
  \bibinfo{volume}{50}(\bibinfo{number}{2}), pp. \bibinfo{pages}{251--260},
  \doi{10.1007/BF00370186}.

\bibitemdeclare{incollection}{Hans96}
\bibitem{Hans96}
\bibinfo{author}{Sven~Ove \surnamestart Hansson\surnameend}
  (\bibinfo{year}{1996}): \emph{\bibinfo{title}{Hidden structures of belief}}.
\newblock In \bibinfo{editor}{Andre \surnamestart Fuhrmann\surnameend} \&
  \bibinfo{editor}{Hans \surnamestart Rott\surnameend}, editors: {\slshape
  \bibinfo{booktitle}{Logic, Actions and Information}}, \bibinfo{publisher}{de
  Gruyter}, pp. \bibinfo{pages}{79--100}.
\newblock
  \urlprefix\url{https://www.degruyter.com/document/isbn/9783110868890/html?lang=en}.

\bibitemdeclare{article}{Hans99b}
\bibitem{Hans99b}
\bibinfo{author}{Sven~Ove \surnamestart Hansson\surnameend}
  (\bibinfo{year}{1999}): \emph{\bibinfo{title}{Recovery and epistemic
  residue}}.
\newblock {\slshape \bibinfo{journal}{Journal of Logic, Language and
  Information}} \bibinfo{volume}{8}, pp. \bibinfo{pages}{421--428},
  \doi{10.1023/A:1008316915066}.

\bibitemdeclare{book}{Hans99c}
\bibitem{Hans99c}
\bibinfo{author}{Sven~Ove \surnamestart Hansson\surnameend}
  (\bibinfo{year}{1999}): \emph{\bibinfo{title}{A textbook of belief dynamics:
  Theory change and database updating}}.
\newblock \bibinfo{publisher}{Springer Dordrecht},
  \bibinfo{address}{Dordrecht}, \doi{10.1007/978-94-007-0814-3}.

\bibitemdeclare{inproceedings}{KonPin17}
\bibitem{KonPin17}
\bibinfo{author}{S{\'e}bastien \surnamestart Konieczny\surnameend} \&
  \bibinfo{author}{Ram{\'o}n \surnamestart Pino~P{\'e}rez\surnameend}
  (\bibinfo{year}{2017}): \emph{\bibinfo{title}{On Iterated Contraction:
  Syntactic Characterization, Representation Theorem and Limitations of the
  Levi Identity}}.
\newblock In \bibinfo{editor}{Seraf{\'i}n \surnamestart Moral\surnameend},
  \bibinfo{editor}{Olivier \surnamestart Pivert\surnameend},
  \bibinfo{editor}{Daniel \surnamestart S{\'a}nchez\surnameend} \&
  \bibinfo{editor}{Nicol{\'a}s \surnamestart Mar{\'i}n\surnameend}, editors:
  {\slshape \bibinfo{booktitle}{Scalable Uncertainty Management}},
  \bibinfo{publisher}{Springer International Publishing}, pp.
  \bibinfo{pages}{348--362}, \doi{10.1007/978-3-319-67582-4_25}.

\bibitemdeclare{book}{Lev91}
\bibitem{Lev91}
\bibinfo{author}{Isaac \surnamestart Levi\surnameend} (\bibinfo{year}{1991}):
  \emph{\bibinfo{title}{The fixation of belief and its undoing}}.
\newblock \bibinfo{publisher}{Cambridge University Press},
  \doi{10.1017/CBO9780511663819}.

\bibitemdeclare{book}{Lev04}
\bibitem{Lev04}
\bibinfo{author}{Isaac \surnamestart Levi\surnameend} (\bibinfo{year}{2004}):
  \emph{\bibinfo{title}{Mild {Contraction}}}.
\newblock \bibinfo{publisher}{Oxford University Press},
  \doi{10.1093/0199270708.001.0001}.

\bibitemdeclare{book}{Lew73}
\bibitem{Lew73}
\bibinfo{author}{David \surnamestart Lewis\surnameend} (\bibinfo{year}{1973}):
  \emph{\bibinfo{title}{Counterfactuals}}.
\newblock \bibinfo{publisher}{Harvard University Press}.
\newblock
  \urlprefix\url{https://www.wiley.com/en-us/Counterfactuals-p-9780631224259}.

\bibitemdeclare{inproceedings}{LinRab91}
\bibitem{LinRab91}
\bibinfo{author}{Sten \surnamestart Lindstr\"om\surnameend} \&
  \bibinfo{author}{Wlodek \surnamestart Rabinowicz\surnameend}
  (\bibinfo{year}{1991}): \emph{\bibinfo{title}{Epistemic entrenchment with
  incomparabilities and relational belief revision}}.
\newblock In \bibinfo{editor}{Andr\'e \surnamestart Fuhrmann\surnameend} \&
  \bibinfo{editor}{Michael \surnamestart Morreau\surnameend}, editors:
  {\slshape \bibinfo{booktitle}{The Logic of Theory Change}},
  \bibinfo{publisher}{Springer}, pp. \bibinfo{pages}{93--126},
  \doi{10.1007/BFb0018418}.

\bibitemdeclare{incollection}{LinRab98}
\bibitem{LinRab98}
\bibinfo{author}{Sten \surnamestart Lindstr\"{o}m\surnameend} \&
  \bibinfo{author}{Wlodek \surnamestart Rabinowicz\surnameend}
  (\bibinfo{year}{1998}): \emph{\bibinfo{title}{Conditionals and the {R}amsey
  Test}}.
\newblock In \bibinfo{editor}{Didier \surnamestart Dubois\surnameend} \&
  \bibinfo{editor}{Henri \surnamestart Prade\surnameend}, editors: {\slshape
  \bibinfo{booktitle}{Belief Change}}, \bibinfo{publisher}{Springer
  Netherlands}, \bibinfo{address}{Dordrecht}, pp. \bibinfo{pages}{147--188},
  \doi{10.1007/978-94-011-5054-5_4}.

\bibitemdeclare{article}{LinRab92}
\bibitem{LinRab92}
\bibinfo{author}{Sten \surnamestart Linstr\"{o}m\surnameend} \&
  \bibinfo{author}{Wlodzimierz \surnamestart Rabinowicz\surnameend}
  (\bibinfo{year}{1992}): \emph{\bibinfo{title}{The {R}amsey test revisited*}}.
\newblock {\slshape \bibinfo{journal}{Theoria}}
  \bibinfo{volume}{58}(\bibinfo{number}{2-3}), pp. \bibinfo{pages}{131--182},
  \doi{10.1111/j.1755-2567.1992.tb01138.x}.

\bibitemdeclare{article}{Mak87}
\bibitem{Mak87}
\bibinfo{author}{David \surnamestart Makinson\surnameend}
  (\bibinfo{year}{1987}): \emph{\bibinfo{title}{On the status of the postulate
  of recovery in the logic of theory change}}.
\newblock {\slshape \bibinfo{journal}{Journal of Philosophical Logic}}
  \bibinfo{volume}{16}, pp. \bibinfo{pages}{383--394},
  \doi{10.1007/BF00431184}.

\bibitemdeclare{article}{Meyetal02}
\bibitem{Meyetal02}
\bibinfo{author}{Thomas \surnamestart Meyer\surnameend},
  \bibinfo{author}{Johannes \surnamestart Heidema\surnameend},
  \bibinfo{author}{Willem \surnamestart Labuschagne\surnameend} \&
  \bibinfo{author}{Louise \surnamestart Leenen\surnameend}
  (\bibinfo{year}{2002}): \emph{\bibinfo{title}{Systematic Withdrawal}}.
\newblock {\slshape \bibinfo{journal}{Journal of Philosophical Logic}}
  \bibinfo{volume}{31}(\bibinfo{number}{5}), pp. \bibinfo{pages}{415--443},
  \doi{10.1023/A:1020199115746}.

\bibitemdeclare{inproceedings}{MosPar92}
\bibitem{MosPar92}
\bibinfo{author}{Lawrence~S \surnamestart Moss\surnameend} \&
  \bibinfo{author}{Rohit \surnamestart Parikh\surnameend}
  (\bibinfo{year}{1992}): \emph{\bibinfo{title}{Topological reasoning and the
  logic of knowledge: preliminary report}}.
\newblock In \bibinfo{editor}{Yoram \surnamestart Moses\surnameend}, editor:
  {\slshape \bibinfo{booktitle}{Proceedings of the 4th Conference on
  Theoretical Aspects of Reasoning about Knowledge (TARK 1992)}},
  \bibinfo{publisher}{Morgan Kaufmann}, pp. \bibinfo{pages}{95-- 105}.

\bibitemdeclare{inproceedings}{Nie91}
\bibitem{Nie91}
\bibinfo{author}{Reinhard \surnamestart Nieder{\'e}e\surnameend}
  (\bibinfo{year}{1991}): \emph{\bibinfo{title}{Multiple contraction a further
  case against {G}\"ardenfors' principle of recovery}}.
\newblock In \bibinfo{editor}{Andr{\'e} \surnamestart Fuhrmann\surnameend} \&
  \bibinfo{editor}{Michael \surnamestart Morreau\surnameend}, editors:
  {\slshape \bibinfo{booktitle}{The Logic of Theory Change}},
  \bibinfo{publisher}{Springer}, pp. \bibinfo{pages}{322--334},
  \doi{10.1007/BFb0018427}.

\bibitemdeclare{incollection}{Ram50}
\bibitem{Ram50}
\bibinfo{author}{Frank~P. \surnamestart Ramsey\surnameend}
  (\bibinfo{year}{1950}): \emph{\bibinfo{title}{General Propositions and
  Causality}}.
\newblock In \bibinfo{editor}{R.~B. \surnamestart Braithwaite\surnameend},
  editor: {\slshape \bibinfo{booktitle}{The Foundations of Mathematics and
  other Logical Essays}}, \bibinfo{publisher}{Humanities Press}, pp.
  \bibinfo{pages}{237--257}, \doi{10.4324/9781315887814}.

\bibitemdeclare{article}{Rot86}
\bibitem{Rot86}
\bibinfo{author}{Hans \surnamestart Rott\surnameend} (\bibinfo{year}{1986}):
  \emph{\bibinfo{title}{Ifs, though, and because}}.
\newblock {\slshape \bibinfo{journal}{Erkenntnis}}
  \bibinfo{volume}{25}(\bibinfo{number}{3}), pp. \bibinfo{pages}{345--370},
  \doi{10.1007/BF00175348}.

\bibitemdeclare{article}{Rot17}
\bibitem{Rot17}
\bibinfo{author}{Hans \surnamestart Rott\surnameend} (\bibinfo{year}{2017}):
  \emph{\bibinfo{title}{Preservation and postulation: lessons from the new
  debate on the {R}amsey test}}.
\newblock {\slshape \bibinfo{journal}{Mind}}
  \bibinfo{volume}{126}(\bibinfo{number}{502}), pp. \bibinfo{pages}{609--626},
  \doi{10.1093/mind/fzw028}.

\bibitemdeclare{article}{PagRot99}
\bibitem{PagRot99}
\bibinfo{author}{Hans \surnamestart Rott\surnameend} \&
  \bibinfo{author}{Maurice \surnamestart Pagnucco\surnameend}
  (\bibinfo{year}{1999}): \emph{\bibinfo{title}{Severe Withdrawal (and
  Recovery)}}.
\newblock {\slshape \bibinfo{journal}{Journal of Philosophical Logic}}
  \bibinfo{volume}{28}(\bibinfo{number}{5}), pp. \bibinfo{pages}{501--547},
  \doi{10.1023/A:1004344003217}.

\bibitemdeclare{inproceedings}{Sauetal20}
\bibitem{Sauetal20}
\bibinfo{author}{Kai \surnamestart Sauerwald\surnameend},
  \bibinfo{author}{Gabriele \surnamestart Kern-Isberner\surnameend} \&
  \bibinfo{author}{Christoph \surnamestart Beierle\surnameend}
  (\bibinfo{year}{2020}): \emph{\bibinfo{title}{A conditional perspective for
  iterated belief contraction}}.
\newblock In \bibinfo{editor}{G.D.~Giacomo \surnamestart et~al\surnameend},
  editor: {\slshape \bibinfo{booktitle}{ECAI 2020}}, \bibinfo{publisher}{IOS
  Press}, \bibinfo{address}{Berlin, Heidelberg}, pp. \bibinfo{pages}{889--896}.
\newblock \urlprefix\url{10.3233/FAIA200180}.

\bibitemdeclare{incollection}{Stal68}
\bibitem{Stal68}
\bibinfo{author}{Robert \surnamestart Stalnaker\surnameend}
  (\bibinfo{year}{1968}): \emph{\bibinfo{title}{A theory of conditionals}}.
\newblock In \bibinfo{editor}{N.~\surnamestart Rescher\surnameend}, editor:
  {\slshape \bibinfo{booktitle}{Studies in logical theory}},
  \bibinfo{publisher}{Blackwell}, pp. \bibinfo{pages}{98--112},
  \doi{10.1007/978-94-009-9117-0_2}.

\end{thebibliography}
\end{document}